\documentclass[letterpaper, 10 pt, conference]{ieeeconf}  

\IEEEoverridecommandlockouts                              

\usepackage{amsthm}
\usepackage{amsmath,amsfonts,amssymb}
\usepackage{graphicx}
\usepackage{algorithm}
\usepackage{algorithmic}
\usepackage{url}
\usepackage{booktabs}
\usepackage{mathtools}
\usepackage{subcaption}
\usepackage{flushend}

\DeclareMathOperator*{\argmin}{arg\,min}
\theoremstyle{plain}
\newtheorem{theorem}{Theorem}[section]
\newtheorem{proposition}[theorem]{Proposition}

\theoremstyle{definition}

\theoremstyle{remark}

\title{\LARGE \bf
Advantage-Guided Diffusion for Model-Based Reinforcement Learning
}

\author{Daniele Foffano$^{1}$, Arvid Eriksson$^{1}$, David Broman$^{2}$, Karl H. Johansson$^{1}$, Alexandre Proutiere$^{1}$
\thanks{$^{1}$Daniele Foffano, Arvid Eriksson, Karl H. Johansson and Alexandre Proutiere are in the Division of Decision and Control Systems of the EECS
School at KTH Royal Institute of Technology, Stockholm, Sweden. {\tt\small \{foffano, arveri, kallej, alepro\}@kth.se}}%
\thanks{$^{2}$David Broman is in the Division of Software and Computer Systems of the EECS
School at KTH Royal Institute of Technology, Stockholm, Sweden.
        {\tt\small dbro@kth.se}\newline All authors are affiliated with Digital Futures.}%
}

\begin{document}

\maketitle
\thispagestyle{empty}
\pagestyle{empty}

\begin{abstract}
Model-based reinforcement learning (MBRL) with autoregressive world models suffers from compounding errors, whereas diffusion world models mitigate this by generating trajectory segments jointly. However, existing diffusion guides are either policy-only—discarding value information—or reward-based, which becomes myopic when the diffusion horizon is short. We introduce \textbf{Advantage-Guided Diffusion} for MBRL (AGD-MBRL), which steers the reverse diffusion process using the agent’s advantage estimates so that sampling concentrates on trajectories expected to yield higher long-term return beyond the generated window. We develop two guides: (i) \emph{Sigmoid Advantage Guidance (SAG)} and (ii) \emph{Exponential Advantage Guidance (EAG)}. We prove that a diffusion model guided through SAG or EAG allows us to perform reweighted sampling of trajectories with weights increasing in state–action advantage—implying policy improvement under standard assumptions. Additionally, we show that the trajectories generated from AGD-MBRL follow an improved policy (that is, with higher value) compared to an unguided diffusion model. AGD integrates seamlessly with PolyGRAD-style architectures by guiding the \emph{state} components while leaving action generation policy-conditioned, and requires no change to the diffusion training objective. On MuJoCo control tasks (HalfCheetah, Hopper, Walker2D and Reacher), AGD-MBRL improves sample efficiency and final return over PolyGRAD, an online Diffuser-style reward guide, and model-free baselines (PPO/TRPO), in some cases by a margin of $2\times$. These results show that advantage-aware guidance is a simple, effective remedy for short-horizon myopia in diffusion-model MBRL.
\end{abstract}

\section{Introduction}

Reinforcement learning (RL) has produced agents that exceed human performance on a diverse range of domains \cite{mirhoseini2021graph, vinyals2019grandmaster, silver2017mastering, silver2016mastering, mnih2015human}. Model-based RL (MBRL) improves sample efficiency by learning a world model and planning within it. However, one-step \emph{autoregressive} world models suffer from compounding errors: the model predicts state $s_{t+1}$ from previous state and action $(s_t,a_t)$, then conditions on its own predictions thereafter; small one-step errors accumulate then over long horizons, leading to performance degradation. Prior work in the literature focuses on maintaining the autoregressive nature while trying to mitigate the compounding error via ensembles \cite{kurutach2018model}, stronger function approximation \cite{micheli2022transformers, robine2023transformer, schubert2023generalist}, or planning in latent space \cite{hafner2023mastering}.

Diffusion models offer a complementary approach: rather than predicting one step at a time, a denoising diffusion model learns a distribution over entire trajectories, sampled through an iterative denoising procedure \cite{janner2022planning, rigter2023world}. Since the diffusion model generates all the steps at the same time, the compounding error effect is drastically reduced.

In practice, diffusion world models are trained with a denoising score-matching objective (equivalent to maximum likelihood under Gaussian perturbations). In other words, they learn to mimic a trajectory distribution induced by a given behavior policy. A diffusion model will therefore generate trajectories similar to the one it was trained on. However, it might be useful to steer the sampling process towards specific areas of the distribution domain. For example, by sampling trajectories that are close to the current estimated policy \cite{rigter2023world,jackson2024policy}, that have higher cumulative reward \cite{janner2022planning}, or that are more challenging for the RL agent \cite{foffano2025adversarial}.

Steering the sampling process can be achieved using a \emph{guided diffusion model}. For instance, PolyGRAD \cite{rigter2023world} integrates a policy guide that favors trajectories following the current policy, effectively making the training algorithm on-policy. This guiding style does not leverage the value functions estimated by the RL agent, therefore discarding important information for the underlying sequential decision-making problem. On the other hand, Diffuser \cite{janner2022planning} steers sampling towards high-return trajectories through reward-based guidance. However, since this guide is based solely on the rewards of the generated steps, it can be myopic if the diffusion trajectory samples have a short horizon, as it is often the case for diffusion MBRL. This happens because the cumulative reward over a short horizon does not account for the value of future states, beyond the generated window, and ignores the long-term return of the current policy.

\textbf{Contributions.} We propose Advantage-Guided Diffusion for MBRL (AGD-MBRL). Using the advantage function learned by the RL agent, we can steer the generative process towards trajectories where the agent is expecting to learn more useful information. In addition, because advantage is formulated through expected values of current and future states, it accounts for rewards \emph{beyond} the sampled horizon. We introduce two guidance mechanisms:
(i) \emph{Exponential Advantage Guidance (EAG)}, where we use exponential tilting to generate trajectories with a higher advantage; and
(ii) \emph{Sigmoid Advantage Guidance (SAG)}, where we define the likelihood of a trajectory to provide useful information using the sigmoid function of the advantage function. Both variants can easily be integrated with PolyGRAD-style architectures by applying guidance to the \emph{state} components of the reverse diffusion process.

We make the following specific contributions:
\begin{enumerate}
    \item We show why steering sampling using cumulative returns results in myopic data generation and why the advantage function can help solving this problem. To this end, we introduce Advantage-Guided Diffusion (AGD) with two different formulations: EAG and SAG (Section \ref{sec:Advantage_guided_diff}).
    \item We formally show that guiding the diffusion process with EAG and SAG is equivalent to perform reweighted sampling of trajectories generated by an improved policy, compared to one of an unguided diffusion model. The weights are higher for trajectories encountering states with higher advantage (Propositions \ref{prop:SAG_improve} and \ref{prop:EAG_improve}). These formal results provide a principled explanation for the faster and more stable learning of AGD-MBRL.
    \item We empirically validate our method, AGD-MBRL, on Gym/MuJoCo tasks; we show improved performance, sample efficiency, and final return compared to other diffusion model-based and model-free baselines (Section \ref{sec:experiments}).
\end{enumerate}

\section{Related work}

We review prior work on MBRL and diffusion models for trajectory generation. We focus on approaches that mitigate compounding errors and on guided-diffusion methods most closely related to our setting.

\textbf{Model-Based reinforcement learning.} In MBRL, the agent uses a world model to generate imaginary trajectories, allowing for planning with reduced interactions with the environment. This is essential for settings where collecting new data is too expensive, dangerous, illegal, or otherwise impractical. Deep neural networks are effective to approximate the environment dynamics \cite{nagabandi2018neural, chua2018deep, kaiser2019model, jafferjee2020hallucinating, van2020plannable}. Specifically, variational autoencoders \cite{kingma2013auto} and transformers \cite{vaswani2017attention} have been successfully employed for world modeling \cite{micheli2022transformers, robine2023transformer, schubert2023generalist, hafner2020mastering, ha2018recurrent}, leading to state-of-the-art methods in terms of sample efficiency and performance \cite{hafner2023mastering}. However, these methods rely on bootstrapping to generate trajectory samples. The state prediction generated by the model is fed again as input to the model to predict the next state. As a result, these methods accumulate error along a trajectory, compounding prediction inaccuracies over time. This is commonly known as the \emph{compounding error problem} of model-based methods. Several works in the literature have addressed this problem \cite{kurutach2018model, hafner2023mastering, xiao2019learning}, trying to reduce compounding errors by focusing on minimizing the step-wise error or adapting the horizon so that the prediction error lies within an acceptable range.

Another line of work revolves around multi-step prediction, where multiple steps are predicted simultaneously. One such example is learning $H$ models to look $H$ steps in the future \cite{asadi2019combating}, but this approach results in significantly higher learning complexity compared to single-model approaches. Only recently, diffusion models have emerged as efficient multi-step MBRL methods \cite{janner2022planning, rigter2023world}.

\textbf{Diffusion models in RL.} Diffusion models generate data through an iterative denoising process \cite{sohl2015deep, ho2020denoising}. In addition to being powerful function approximators, they also provide a natural way to condition the data generation based on other information useful to the prediction task, through guided diffusion \cite{dhariwal2021diffusion, ho2022classifier}. Recently, diffusion models have gained significant attention in the RL community, where they have been used to learn a model of the system dynamics and generate trajectory segments by predicting sequences of states \cite{ajay2022conditional, zhu2023madiff}, actions \cite{chi2023diffusion, li2024crossway}, or both \cite{janner2022planning, liang2023adaptdiffuser, jackson2024policy}. Additionally, diffusion models have been employed for policy modeling \cite{wang2019benchmarking, hansen2023idql} and value function approximation \cite{mazoure2023value}. While most prior work focuses on the offline RL setting, we develop an online method. Our implementation directly builds on the PolyGRAD architecture \cite{rigter2023world}, a Dyna-inspired \cite{sutton1991dyna} MBRL method using policy-guided diffusion to approximate the environment dynamics.

\section{Preliminaries}
\label{sec:preliminaries}

We formalize the online MBRL setup, introduce notation for value and advantage functions, and review the guided-diffusion framework on which AGD-MBRL builds.

\textbf{Markov decision processes and RL.} Consider a Markov decision process (MDP) $M = \langle S, A, P, r, \gamma, \rho \rangle$, where $S$ and $A$ are the state and action spaces, respectively, $P(\cdot |s,a)$ is the transition probability function, $r(s,a)$ the reward function, $\gamma$ the discount factor, and $\rho$ the initial state distribution. By interacting with the MDP, an RL agent is able to collect sequences of states, actions, and rewards, forming trajectories $\boldsymbol \tau = (s_0, a_0, r_0, \dots, s_H, a_H, r_H)$ over a horizon $H>0$. The objective of the RL agent is to learn an optimal policy $\pi^\star$ maximizing the policy value $V_\pi(s) = \mathbb{E}_{\pi}[\sum_{i=0}^{\infty}\gamma^ir_{t+i+1}|s_t = s]$. Alternatively, we can define the optimal policy as the one maximizing the $Q$-value function $Q_\pi(s,a) = \mathbb{E}_{\pi}[\sum_{i=0}^{\infty}\gamma^ir_{t+i+1}|s_t = s, a_t = a]$. Several Policy Gradient methods combine the two value functions into a new objective \cite{schulman2015trust,schulman2017proximal}, the advantage function, defined as $A_\pi(s,a) = Q_\pi(s,a) - V_\pi(s)$. The advantage function can be considered as a state--action value function using $V_\pi$ as a baseline to reduce the variance. The value of a policy $\pi$ can be defined as $J(\pi) = \mathbb{E}_{s\sim\rho}[V_\pi(s)]$. In this paper, we consider an MBRL setting, where we use a diffusion model to approximate the distribution of trajectories under a given policy. Specifically, if $p^\pi$ denotes the true distribution of the trajectories $\boldsymbol{\tau}$ under policy $\pi$, the diffusion model samples trajectories with distribution $p_{\boldsymbol{\theta}}$ close to $p^\pi$. We adopt a Dyna-style approach \cite{sutton1991dyna}, where the diffusion model and the policy are iteratively updated. The diffusion model is updated using samples gathered from the target environment under the current policy, while the policy is improved using data generated by the model.

\textbf{Diffusion models.}\label{subsec:diff} To generate trajectories we use a parametrized diffusion model $p_{\boldsymbol \theta}(\boldsymbol\tau_0)$, from which we sample synthetic trajectories to perform policy training.

The generative process of diffusion models is performed by progressively refining noisy inputs through iterative denoising steps, $p_{\boldsymbol{\theta}}(\boldsymbol{\tau}_{i-1}| \boldsymbol{\tau}_i)$, which reverse the forward diffusion process, $q(\boldsymbol{\tau}_{i}| \boldsymbol{\tau}_{i-1})$. Each forward step gradually corrupts real data by adding random Gaussian noise. The $i$-th denoising step is typically parameterized as a Gaussian distribution
\begin{align}\label{eq:gausdiff}
    p_{\boldsymbol{\theta}}(\boldsymbol{\tau}_{i-1} \vert \boldsymbol{\tau}_i) = \mathcal{N}(\boldsymbol{\mu}_{\boldsymbol \theta}(\boldsymbol \tau_i, i), \boldsymbol \Sigma_i),
\end{align}
with learned mean and fixed covariance.

The joint probability of the denoised trajectories is defined as 
\begin{align}
    p_{\boldsymbol{\theta}}(\boldsymbol{\tau}_{0:N}) = p(\boldsymbol{\tau}_N) \prod^N_{i=1} p_{\boldsymbol{\theta}}(\boldsymbol{\tau}_{i-1} \vert \boldsymbol{\tau}_i),
\end{align}
where $p(\boldsymbol \tau_N) \approx \mathcal{N}(\boldsymbol 0, \boldsymbol I)$ and $\boldsymbol \tau_0$ is the reconstructed noiseless trajectory. Diffusion models optimize the variational lower bound on the negative log likelihood:
\begin{align}
\label{eq:vlb}
    \boldsymbol \theta^\star = \argmin_{\boldsymbol \theta} \mathbb{E}_{\boldsymbol \tau_0}[-\log p_{\boldsymbol \theta}(\boldsymbol \tau_0)],
\end{align}
where $p_{\boldsymbol \theta}(\boldsymbol \tau_0) = \int p_{\boldsymbol{\theta}}(\boldsymbol{\tau}_{0:N})\text{d}\boldsymbol{\tau}_{1:N}$.

\textbf{Guided diffusion.} By introducing an additional classifier $p(y|\boldsymbol \tau_0)$ we define the conditional diffusion model
\begin{align}
\label{eq:guided_diff}
    p_{\boldsymbol\theta}(\boldsymbol\tau_0|y) \propto p_{\boldsymbol\theta}(\boldsymbol\tau_0)p(y|\boldsymbol \tau_0).
\end{align}
By leveraging the information provided by the classifier, we can improve the generative performance of the diffusion model \cite{dhariwal2021diffusion}. Through the classifier’s gradient, we can guide the denoising process toward sampling data points that align with the classifier’s output. This method, called classifier-guided diffusion, follows a denoising procedure with the following step
\begin{align}
    p_{\boldsymbol{\theta}}(\boldsymbol{\tau}_{i-1}| \boldsymbol{\tau}_i, y) = \mathcal{N}(\mu_{\boldsymbol{\theta}}(\boldsymbol \tau_i,i) + \alpha\boldsymbol{\Sigma}_i \boldsymbol{g}_i, \boldsymbol{\Sigma}_i)
    \label{eq:gradient_guided_diffusion}
\end{align}
where $\boldsymbol{g}_i = \nabla_{\boldsymbol{\tau}}\log p(y|\boldsymbol{\tau})|_{\boldsymbol{\tau}=\mu_{\boldsymbol{\theta}}(\boldsymbol \tau_i,i)}$ and $\alpha$ is a hyperparameter determining the intensity of the guide.

Guided diffusion has received growing attention in the RL community, finding several applications both in the online and offline setting \cite{janner2022planning, rigter2023world, 
foffano2025adversarial, zheng2024safe, shribak2024diffusion, ki2025prior, ma2025efficient}.

\textbf{Problem statement.} In MBRL, an RL agent learns the optimal policy using data generated by an MDP model. In this paper, we focus on using a diffusion model $p_{\boldsymbol\theta}(\boldsymbol\tau_0)$ to approximate the trajectory distribution under a given policy $\pi$, that is, $ p(\boldsymbol\tau_0) = p_{\boldsymbol\theta}(s_1)\prod_{t=1}^H p(s_{t+1} | s_t, a_t) \pi(a_t |s_t)$. Generating data according to the model $p_{\boldsymbol\theta}(\boldsymbol\tau_0) \approx p(\boldsymbol\tau_0)$ mimics generating trajectories under the policy $\pi$. We are interested in how to improve the learning process of the RL agent by leveraging information provided by the advantage function $A_\pi(s,a)$. Using classifier-guided diffusion, we can steer the sampling process towards trajectories with a higher advantage, where the RL agent is expected to learn information that is more useful to converge to the optimal policy.

With AGD-MBRL, we address the problem of steering the diffusion sampling process of short-horizon trajectories. Given a policy $\pi$ and its advantage $A_\pi$, we steer the diffusion process toward trajectory segments with higher long-term value beyond the diffusion horizon, yielding more informative data for improving $\pi$.

\section{Motivating example}
\label{sec:example}
\begin{figure}
    \centering
    \includegraphics[width=\linewidth]{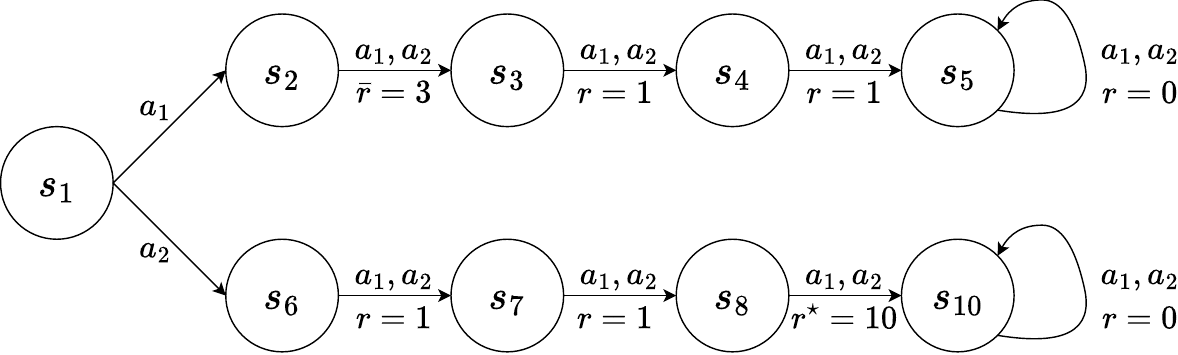}
    \caption{MDP illustrating the negative effects of short-sighted planning.}
    \label{fig:MDP}
\end{figure}
To motivate our contribution, we study the diffusion guide proposed in the Diffuser architecture \cite{janner2022planning} on the MDP presented in Figure \ref{fig:MDP}. In Diffuser, the authors rely on the control-as-inference graphical model \cite{levine2018reinforcement}. By introducing a binary random variable $O_t$, they define the probability of a single trajectory step $(s_t, a_t, r_t)$ being optimal\footnote{Meaning that taking action $a_t$ when observing $s_t$ leads to the highest reward $r_t$.} as $p(O_t = 1 | s_t, a_t) = \exp{(r_t)}$. The optimality for an entire trajectory $\boldsymbol\tau$ is then naturally defined as $p(O_{0:H} = 1 | \boldsymbol\tau) = \prod_{t=0}^H p(O_t = 1|s_t, a_t) = \exp{(\sum_{t=0}^H r_t)}$. The classifier-guided diffusion model can then be formulated according to (\ref{eq:guided_diff}), taking $p(y|\boldsymbol\tau_0) = p(O_{0:H} = 1 | \boldsymbol\tau_0)$ and $\boldsymbol g_i = \sum_{t=1}^H\nabla_{(s_t,a_t)} r(s_t, a_t)|_{(s_t,a_t) \in \mu_\theta(\tau_i, i)}$.  Using the cumulative reward as an estimate for the value of a trajectory works when $H$ corresponds to the real horizon $H^\star$ of the problem. However, this assumption does not hold for most diffusion applications in the RL setting, where $H \ll H^\star$ because of high computational costs. In that case, the model might steer the sampling process toward suboptimal areas of the domain by overestimating the value of a trajectory. The guiding method we propose in this paper avoids this by incorporating long-term rewards through the advantage function.

We can see the myopic effect of the guide based on cumulative reward by analyzing the MDP in Figure \ref{fig:MDP}. Suppose the diffusion model has a horizon $H = 3$, and that we follow a policy $\pi$ for which $\pi(a_1 | s_1) = \pi(a_2|s_1) = 0.5$ and then $\pi(a_1|s) =1, \forall s \neq s_1$. Two trajectories can be sampled from the diffusion model when starting in $s_1$: $\boldsymbol\tau_1 = (s_1, a_1, r, s_2, a_1, \bar{r}, s_3, a_1, r)$ and $\boldsymbol\tau_2 = (s_1, a_2, r, s_6, a_1, r, s_7, a_1, r)$. If the guide of our diffusion model steers the sampling towards the trajectory with highest cumulative reward, then we will more likely sample $\boldsymbol\tau_1$. However, the optimal trajectory to follow is $\boldsymbol\tau_2$ due to the unobserved future reward $r^\star = 10$. However, if we look at the advantage of each state--action pair after $4$ steps of value iteration with discount factor $\gamma = 1$, we get $A_\pi(s_1, a_1) = -7$ and $A_\pi(s, a) = 0, \forall(s,a) \neq (s_1, a_1)$. The advantage that is propagated along $\boldsymbol\tau_2$ is therefore much higher than the one along $\boldsymbol\tau_1$. This observation leads to the core idea of our paper: making the guided diffusion process less myopic by leveraging an advantage-based guide. 

\section{Advantage-guided diffusion}
\label{sec:Advantage_guided_diff}
We propose two guidance methods: Sigmoid Advantage Guidance (SAG) and Exponential Advantage Guidance (EAG).

\textbf{Sigmoid Advantage Guidance.} 
In Diffuser \cite{janner2022planning}, the optimality variable is defined via control-as-inference as a Bernoulli with parameter
$p(O_t{=}1\mid s_t,a_t)=\exp(r_t)$. This identity is exact under the assumption that rewards are shifted so that $r_t\le 0$ (hence $\exp( r_t)\in(0,1]$). This assumption is not easily transferable when defining the Bernoulli probability using the advantage function $A_\pi(s_t, a_t)$, since it can be both positive and negative, and it might not be bounded.

Instead of an exponential density, we instead propose to use the sigmoid function to model the step optimality probability. That is, we define
\begin{align*}
    p(O_t = 1 | s_t, a_t) &= \sigma(A_t) = \frac{1}{1+\exp{(-A_t)}}\\
    p(O_{1:H} = 1 | \boldsymbol\tau) &= \prod_{t=1}^H p(O_t | s_t, a_t) = \frac{1}{\prod_{t=1}^H (1+\exp{(-A_t)})}
\end{align*}
where, with a slight abuse of notation, we have $A_t = A_\pi(s_t, a_t)$. If we define a diffusion model $p_{\boldsymbol\theta}(\boldsymbol\tau_0)$ as in Section \ref{sec:preliminaries}, we can use classifier guided diffusion to derive a model $p_{\boldsymbol\theta}(\boldsymbol\tau_0 | A_\pi)$ that approximates the distribution of trajectories with high advantage as
\begin{align*}
    p_{\boldsymbol\theta} (\boldsymbol\tau_0 | O_{1:H} = 1) \propto p_{\boldsymbol\theta}(\boldsymbol\tau_0)p(O_{1:H} = 1| \boldsymbol\tau_0),
\end{align*}
where the denoising steps are taken according to (\ref{eq:gradient_guided_diffusion}) with $\boldsymbol{g}_i = \nabla_{\boldsymbol\tau} \log p(O_{1:H} = 1| \boldsymbol\tau)|_{\boldsymbol\tau = \mu_{\boldsymbol\theta}(\boldsymbol\tau_i, i)} = \sum^H_{t=1}\frac{1}{1+\exp{(A_\pi(s_t, a_t))}}\nabla_{(s_t,a_t)}A_\pi(s_t, a_t)$. Note that the right-hand side is proportional up to a normalization constant.

We now proceed to show that, given a diffusion model $p_{\boldsymbol\theta}$ approximating the trajectory distribution under policy $\pi_{\boldsymbol\theta}$, using SAG to steer the sampling process is equivalent to perform reweighted sampling of trajectories generated using a policy $\pi_{\boldsymbol\theta}^\sigma$ for which $J(\pi_{\boldsymbol\theta}^\sigma) \geq J(\pi_{\boldsymbol\theta})$, with a higher weight for trajectories that encounter states with a higher advantage. 

\begin{proposition}
\label{prop:SAG_improve}
    Define a diffusion model $p_{\boldsymbol\theta}(\boldsymbol\tau_0)$ as in Section \ref{sec:preliminaries}, approximating the trajectory distribution under policy $\pi$. Then, sampling trajectories from 
    \begin{align*}
        p_{\boldsymbol\theta} (\boldsymbol\tau_0 | O_{1:H}) \propto p_{\boldsymbol\theta}(\boldsymbol\tau_0)p(O_{1:H} | \boldsymbol\tau_0)
    \end{align*} 
    is equivalent to perform reweighted sampling with weights $Z_{\boldsymbol\theta}(\boldsymbol\tau) = \prod_{t=1}^H \mathbb{E}_{a\sim\pi_{\boldsymbol\theta}(\cdot | s_t)}\left[\sigma(A_{\pi_{\boldsymbol\theta}}(s_t, a))\right]$ of trajectories generated under a policy $\pi_{\boldsymbol\theta}^\sigma$ for which it holds $J(\pi^\sigma_{\boldsymbol\theta}) \geq J(\pi_{\boldsymbol\theta})$.
\end{proposition}
\begin{proof}
    Since $p_{\boldsymbol\theta}(\boldsymbol\tau_0)$ approximates the trajectory distribution under policy $\pi_{\boldsymbol\theta}$, we can write
    \begin{align*}
        p_{\boldsymbol\theta}(\boldsymbol\tau_0) = p_{\boldsymbol\theta}(s_1)\prod_{t=1}^H p_{\boldsymbol\theta}(s_{t+1} | s_t, a_t) \pi_{\boldsymbol\theta}(a_t |s_t)
    \end{align*}
    So when steering the sampling with SAG we obtain
    \begin{align*}
        p_{\boldsymbol\theta} (&\boldsymbol\tau_0 | O_{1:H}) \propto p_{\boldsymbol\theta}(\boldsymbol\tau_0)p(O_{1:H} | \boldsymbol\tau_0)\\
        &= p(O_{1:H} | \boldsymbol\tau_0)p_{\boldsymbol\theta}(s_1)\prod_{t=1}^H p_{\boldsymbol\theta}(s_{t+1} | s_t, a_t) \pi_{\boldsymbol\theta}(a_{t} |s_{t})\\
        &= p_{\boldsymbol\theta}(s_1)\prod_{t=1}^H p_{\boldsymbol\theta}(s_{t+1} | s_t, a_t) \pi_{\boldsymbol\theta}(a_{t} |s_{t})p(O_{t} | s_{t}, a_{t})\\
        &= p_{\boldsymbol\theta}(s_1)\prod_{t=1}^H p_{\boldsymbol\theta}(s_{t+1} | s_t, a_t) \pi^\sigma_{\boldsymbol\theta}(a_{t} |s_{t})Z_{\boldsymbol\theta}(s_{t})
    \end{align*}
    where $\pi^\sigma_{\boldsymbol\theta}(a_t|s_t) = \frac{1}{Z_\theta(s_t)} \pi_{\boldsymbol\theta}(a_{t} |s_{t})p(O_{t} | s_{t}, a_{t})$ and $Z_{\boldsymbol\theta}(s_t) = \mathbb{E}_{a\sim\pi_{\boldsymbol\theta}(\cdot | s_t)}\left[\sigma(A_{\pi_{\boldsymbol\theta}}(s_t, a))\right]$.

    From the policy improvement theorem \cite[Section 4.2]{sutton1998reinforcement} we know that if $\mathbb{E}_{a\sim\pi^\sigma_{\boldsymbol\theta}(\cdot | s_t)}\left[A_{\pi_{\boldsymbol\theta}}(s_t, a) \right] \geq 0$ then $J(\pi^\sigma_{\boldsymbol\theta}) \geq J(\pi_{\boldsymbol\theta})$. We can rewrite the expectation as
    \begin{align*}
        \mathbb{E}_{a\sim\pi^\sigma_{\boldsymbol\theta}}&\left[A_{\pi_{\boldsymbol\theta}}(s_t, a) \right] \\
        &= \frac{\mathbb{E}_{a\sim\pi_{\boldsymbol\theta}(\cdot | s_t)}\left[\sigma(A_{\pi_{\boldsymbol\theta}}(s_t, a))A_{\pi_{\boldsymbol\theta}}(s_t, a) \right]}{\mathbb{E}_{a\sim\pi_{\boldsymbol\theta}(\cdot | s_t)}\left[\sigma(A_{\pi_{\boldsymbol\theta}}(s_t, a))\right]}
    \end{align*}
    Since $\mathbb{E}_{a\sim\pi_{\boldsymbol\theta}(\cdot | s_t)}\left[A_{\pi_{\boldsymbol\theta}}(s_t, a)\right] = 0$ by definition of the advantage, then $\mathbb{E}_{a\sim\pi_{\boldsymbol\theta}(\cdot | s_t)}\left[\sigma(A_{\pi_{\boldsymbol\theta}}(s_t, a))A_{\pi_{\boldsymbol\theta}}(s_t, a) \right] \geq 0$ by Chebyshev's integral inequality. By definition of the sigmoid function we also have $\mathbb{E}_{a\sim\pi_{\boldsymbol\theta}(\cdot | s_t)}\left[\sigma(A_{\pi_{\boldsymbol\theta}}(s_t, a))\right] \geq 0$, so $\mathbb{E}_{a\sim\pi^\sigma_{\boldsymbol\theta}(\cdot | s_t)}\left[A_{\pi_{\boldsymbol\theta}}(s_t, a) \right] \geq 0$ holds. By the policy improvement theorem we then have $J(\pi^\sigma_{\boldsymbol\theta}) \geq J(\pi_{\boldsymbol\theta})$, which concludes the proof. 
\end{proof}

\textbf{Exponential Advantage Guidance.} When dealing with trajectory generation, defining optimality based on a single step is not straightforward. Here, we follow an energy-based approach to steer the generative process towards trajectories with high advantage. More specifically, define the energy function $E(\boldsymbol\tau) = \sum_{t=1}^H A_\pi(s_t, a_t)$. Then, we can define the energy-guided diffusion model $ p_{\boldsymbol\theta}(\boldsymbol\tau_0 | E(\boldsymbol\tau_0)) \propto p_{\boldsymbol\theta}(\boldsymbol\tau_0)\exp{(E(\boldsymbol\tau_0))}$. Intuitively, the sampling process of this energy-guided diffusion model is steered towards trajectories with a high cumulative advantage. That is, we will generate more trajectories where the agent is expected to improve its behavior towards an optimal solution. Using the guided diffusion framework, as we did for SAG, we can sample from $p_{\boldsymbol\theta}(\boldsymbol\tau_0 | E(\boldsymbol\tau_0))$ by taking diffusion steps according to (\ref{eq:gradient_guided_diffusion}) with $\boldsymbol g_i = \nabla_{\boldsymbol\tau} \log \exp{(E(\boldsymbol\tau))} = \nabla_{\boldsymbol\tau} \sum_{t=1}^H A_\pi(s_t, a_t) = \sum_{t=1}^H \nabla_{(s_t, a_t)}A_\pi(s_t, a_t)$.
Following an argument similar to Proposition \ref{prop:SAG_improve}, we can show that using EAG to steer the generative process leads to a reweighted sampling method for trajectories generated with an improved policy $\pi^{\exp}$ for which $J(\pi^{\exp}_{\boldsymbol\theta}) \geq J(\pi_{\boldsymbol\theta})$.

\begin{proposition}
\label{prop:EAG_improve}
    Define a diffusion model $p_{\boldsymbol\theta}(\boldsymbol\tau_0)$ as in Section \ref{sec:preliminaries}, approximating the trajectory distribution under policy $\pi$. Then, sampling trajectories from 
    \begin{align*}
        p_{\boldsymbol\theta}(\boldsymbol\tau_0 | E(\boldsymbol\tau_0)) \propto p_{\boldsymbol\theta}(\boldsymbol\tau_0)\exp{(E(\boldsymbol\tau_0))}
    \end{align*} 
    is equivalent to perform reweighted sampling with weights $Z_{\boldsymbol\theta}(\boldsymbol\tau) = \prod_{t=1}^H \mathbb{E}_{a\sim\pi_{\boldsymbol\theta}(\cdot | s_t)}\left[\exp(A_{\pi_{\boldsymbol\theta}}(s_t, a)))\right]$ of trajectories generated under a policy $\pi_{\boldsymbol\theta}^{\exp}$ for which it holds $J(\pi^{\exp}_{\boldsymbol\theta}) \geq J(\pi_{\boldsymbol\theta})$.
\end{proposition}
\begin{proof}
    As in Proposition \ref{prop:SAG_improve}, we write
    \begin{align*}
        p_{\boldsymbol\theta}(&\boldsymbol\tau_0 | E(\boldsymbol\tau_0)) \propto p_{\boldsymbol\theta}(\boldsymbol\tau_0)\exp{(E(\boldsymbol\tau_0))}\\
        &= \exp{(E(\boldsymbol\tau_0))}p_{\boldsymbol\theta}(s_1)\prod_{t=1}^H p_{\boldsymbol\theta}(s_{t+1} | s_t, a_t) \pi_{\boldsymbol\theta}(a_{t} |s_{t})\\
        &= p_{\boldsymbol\theta}(s_1)\prod_{t=1}^H p_{\boldsymbol\theta}(s_{t+1} | s_t, a_t) \pi_{\boldsymbol\theta}(a_{t} |s_{t})\exp{(A^{\pi_{\boldsymbol\theta}}(s_t, a_t))}\\
        &= p_{\boldsymbol\theta}(s_1)\prod_{t=1}^H p_{\boldsymbol\theta}(s_{t+1} | s_t, a_t) \pi^{\exp}_{\boldsymbol\theta}(a_{t} |s_{t})Z_{\boldsymbol\theta}(s_t)
    \end{align*}
    where $\pi^{\exp}_{\boldsymbol\theta}(a_{t} |s_{t}) = \frac{1}{Z_{\boldsymbol\theta}(s_t)}\pi_{\boldsymbol\theta}(a_{t} |s_{t})\exp{(A_{\pi_{\boldsymbol\theta}}(s_t, a_t))}$ and $Z_{\boldsymbol\theta}(s_t) = \mathbb{E}_{a \sim \pi_{\boldsymbol\theta}(\cdot | s_t)}\left[ \exp{(A_{\pi_{\boldsymbol\theta}}(s_t, a))}\right]$. Following the same reasoning as in Proposition \ref{prop:SAG_improve}, we have 
    \begin{align*}
        \mathbb{E}_{a\sim\pi^{\exp}_{\boldsymbol\theta}(\cdot | s_t)}&\left[A_{\pi_{\boldsymbol\theta}}(s_t, a_t) \right] = \\ 
        &= \frac{\mathbb{E}_{a\sim\pi_{\boldsymbol\theta}(\cdot | s_t)}\left[\exp(A_{\pi_{\boldsymbol\theta}}(s_t, a))A_{\pi_{\boldsymbol\theta}}(s_t, a) \right]}{\mathbb{E}_{a\sim\pi_{\boldsymbol\theta}(\cdot | s_t)}\left[\exp(A_{\pi_{\boldsymbol\theta}}(s_t, a))\right]},
    \end{align*}
    and $\mathbb{E}_{a\sim\pi^{\exp}_{\boldsymbol\theta}(\cdot | s_t)}\left[A_{\pi_{\boldsymbol\theta}}(s_t, a_t) \right]\geq 0$. So, by the policy improvement theorem, we have $J(\pi^{\exp}_{\boldsymbol\theta}) \geq J(\pi_{\boldsymbol\theta})$, concluding the proof. 
\end{proof}

\textbf{Why exponential and sigmoid functions?} 
Both guides apply a monotonically increasing transformation of the advantage $A(s_t,a_t)$ so that the generative process will be steered toward higher-advantage state–action pairs. While other functions can be used to accomplish this task, we chose the \emph{sigmoid} and the \emph{exponential} for their simple implementation and their statistical and optimization properties.

\paragraph{Sigmoid function} through the sigmoid function we are able to express the probability of a step being optimal as a Bernoulli probability $p(O_t | s_t, a_t) = \sigma(A_\pi(s_t, a_t)) = \frac{1}{1+\exp{(-A_\pi(s_t, a_t))}}$. This bounded function leads to a conservative preference weighting for each state-action pair, based on its estimated advantage. For very high advantage values, the sigmoid function will tend to $1$: intuitively, this compensates for an overestimate of the advantage function learned by the RL agent.
\paragraph{Exponential function} this approach follows an opposite direction with respect to that of the sigmoid function. As the probability of generating state-action pairs with a high advantage exponentially increases, it strongly steers the sampling process toward more promising regions of the state-action space. While this method might converge faster to an optimal policy in the case of a good estimate of the advantage function, it is more vulnerable to overestimation (or underestimation).

\section{Algorithms}

\begin{figure*}[t]
\centering
\begin{minipage}[t]{0.48\textwidth}
\begin{algorithm}[H]
\caption{AGD-MBRL}
\label{alg:ADG-MBRL}
\begin{algorithmic}[1]
\REQUIRE environment $E$
\ENSURE policy $\pi_{\boldsymbol\phi}$, advantage function $A_\omega$, denoising model $\bar{p}_{\boldsymbol\theta}$, data buffer $\mathcal{D}$
\STATE \textbf{Initialize:} $\pi_{\boldsymbol\phi}$, $\bar{p}_{\boldsymbol\theta}$, $\mathcal{D}$
\REPEAT
    \STATE Sample $\boldsymbol\tau \sim E$ using $\pi_{\boldsymbol\phi}$, add $\boldsymbol\tau$ to $\mathcal{D}$
    \STATE Improve $p_{\boldsymbol\theta}$ on $\mathcal{D}$
    \STATE Sample $\{\hat{\boldsymbol\tau}\} \sim p_{\boldsymbol\theta}$ with EAG or SAG \hfill $\triangleright$ Algorithm 2
    \STATE Improve $\pi_{\boldsymbol\phi}$ on $\{\hat{\boldsymbol\tau}\}$ using RL
\UNTIL{convergence or budget exhausted}
\end{algorithmic}
\end{algorithm}
\end{minipage}
\hfill
\begin{minipage}[t]{0.48\textwidth}
\begin{algorithm}[H]
\caption{Advantage-Guided Trajectory Sampling}
\label{alg:guided_sampling}
\begin{algorithmic}[1]
\REQUIRE denoising model $p_{\boldsymbol{\theta}}$; steps $N$; buffer $\mathcal{D}$; scale $\alpha$
\ENSURE Sample $\hat{\boldsymbol{\tau}}_0$
\STATE $\hat{\boldsymbol{\tau}}_N \sim \mathcal{N}(\mathbf{0}, \mathbf{I})$
\STATE $\boldsymbol{s}_0 \sim \mathcal{D}$
\FOR{$i = N$ \TO $1$}
    \STATE Set $\hat{\boldsymbol{s}}_0 \gets \boldsymbol{s}_0$ in $\hat{\boldsymbol{\tau}}_i$
    \STATE Set $\boldsymbol{g}_i \gets \nabla_{\hat{\boldsymbol{\tau}}_i}\log f(\hat{\boldsymbol{\tau}}_i)$ \hfill $\triangleright$ EAG or SAG
    \STATE $\hat{\boldsymbol{\tau}}_{i-1} \sim \mathcal{N}\big(\mu_{\boldsymbol{\theta}}(\hat{\boldsymbol{\tau}}_i, i) + \alpha \boldsymbol{\Sigma}_i \boldsymbol{g}_i,\ \boldsymbol{\Sigma}_i\big)$
\ENDFOR
\STATE \textbf{return} $\hat{\boldsymbol{\tau}}_0$
\end{algorithmic}
\end{algorithm}
\end{minipage}
\end{figure*}

In this section we present AGD-MBRL\footnote{The official implementation for AGD-MBRL can be found on GitHub: \url{https://github.com/danielefoffano/AGD-MBRL}}, a Dyna-style algorithm that iteratively improves the dynamics model $p_{\boldsymbol\theta}(\boldsymbol\tau)$ and the RL agent's policy $\pi_{\boldsymbol\phi}$. By leveraging the advantage guides presented in the previous section, AGD-MBRL is able to generate trajectories with a higher advantage, allowing the RL agent to observe trajectories that achieve higher expected rewards more often. Our implementation is based on the PolyGRAD architecture \cite{rigter2023world}.

AGD-MBRL is presented in Algorithm \ref{alg:ADG-MBRL}. Using policy $\pi_{\boldsymbol\phi}$ we collect trajectory rollouts from the real environment (line 3) which are then used to improve the diffusion model $p_{\boldsymbol\theta}$. The diffusion model is trained by minimizing the negative log-likelihood as in (\ref{eq:vlb}). Note, however, that (\ref{eq:vlb}) is a simplified training objective with respect to the one used by diffusion implementations. In practice, the mean $\mu_{\boldsymbol\theta}(\boldsymbol\tau_i, i)$ for (\ref{eq:gausdiff}) is computed in closed form \cite{ho2020denoising} using a noise prediction network $\epsilon_{\boldsymbol \theta}(\hat{\boldsymbol \tau}_i, i)$, trained by minimizing the following loss, derived from (\ref{eq:vlb}),
\begin{align}
\label{eq:diffusion_loss}
    \mathcal{L}(\boldsymbol\theta) = \mathbb{E}_{i, \epsilon, \boldsymbol \tau_0}[||\epsilon - \epsilon_{\boldsymbol \theta}(\boldsymbol \tau_i, i)||^2],
\end{align}
where $i\sim {\cal U}(\{1,\dots, N\})$ is the diffusion step, $\epsilon \sim \mathcal{N}(0, 1)$ is the target noise and $\boldsymbol \tau_i$ is the corrupted data point derived from trajectory $\boldsymbol \tau_0\sim {\cal D}$, after $i$ steps of the \textit{forward} diffusion process adding noise $\epsilon$. With advantage-guided diffusion, we are able to sample new trajectories, on which we will improve our RL agent using Advantage Actor-Critic \cite{mnih2016asynchronous} (respectively line 5 and 6 of Algorithm \ref{alg:ADG-MBRL}). Advantage-guided sampling is presented in Algorithm \ref{alg:guided_sampling}. Notice that in line $5$ we use a general gradient formula $\boldsymbol{g}_i \gets \nabla_{\hat{\boldsymbol{\tau}}_i}\log f(\hat{\boldsymbol{\tau}}_i)$. For EAG and SAG we have, respectively, $f(\hat{\boldsymbol{\tau}}_i) = \exp{(E(\hat{\boldsymbol{\tau}}_i))}$ and $f(\hat{\boldsymbol{\tau}}_i) = p(O_{1:H}|\hat{\boldsymbol{\tau}}_i)$. The sampling procedure follows naturally from classifier guided diffusion theory, as seen in Section \ref{sec:preliminaries} and in Section \ref{sec:Advantage_guided_diff}. At every step of the denoising process we perform inpainting \cite{lugmayr2022repaint} by setting the starting state $\hat s_0$ of the generated trajectory $\hat {\boldsymbol\tau}_i$ to the real state $s_0$. To ensure that the generated actions are consistent with the generated states, we combine our method with the PolyGRAD diffusion guide \cite{rigter2023world}, which generates a sequence of actions guided by the gradient of the policy $\pi_{\boldsymbol\phi}$.

\section{Experiments}
\label{sec:experiments}

We evaluate the performance of AGD-MBRL (with SAG and EAG) on four environments from the continuous control suite of Mujoco, namely: HalfCheetah, Hopper, Walker2D, and Reacher. These are some of the most common environments used to evaluate RL algorithms, and allow us to easily compare our method with a broad range of baselines. The proposed methods and baselines are trained on $1.5$M environment steps. We compare AGD-MBRL with four baselines, equally divided between model-based (with a diffusion model) and model-free methods. We describe them briefly for completeness:
\begin{itemize}
    \item PolyGRAD, \cite{rigter2023world}, a diffusion-based MBRL algorithm, on which AGD-MBRL is based. PolyGRAD uses policy-guided diffusion to generate trajectories, training an RL policy in an online model-based setting. It improves sample efficiency but does not account for the trajectory value when sampling from the diffusion model.
    \item Online Diffuser is our adaptation of the Diffuser architecture \cite{janner2022planning} to the online setting. Leveraging the PolyGRAD architecture, we steer the diffusion sampling process using the Diffuser guide described in Section \ref{sec:example}. This baseline guides the diffusion model using the cumulative reward of the generated trajectory as an estimate of its future value. This might result in myopic planning, as described in Section \ref{sec:example}.
    \item TRPO \cite{schulman2015trust} and PPO \cite{schulman2017proximal}, are two efficient model-free baselines. TRPO constrains policy updates through a KL-divergence trust region. PPO uses a clipped surrogate objective, improving its computational efficiency.
\end{itemize}

\begin{figure*}
    \centering
    \includegraphics[width=\linewidth]{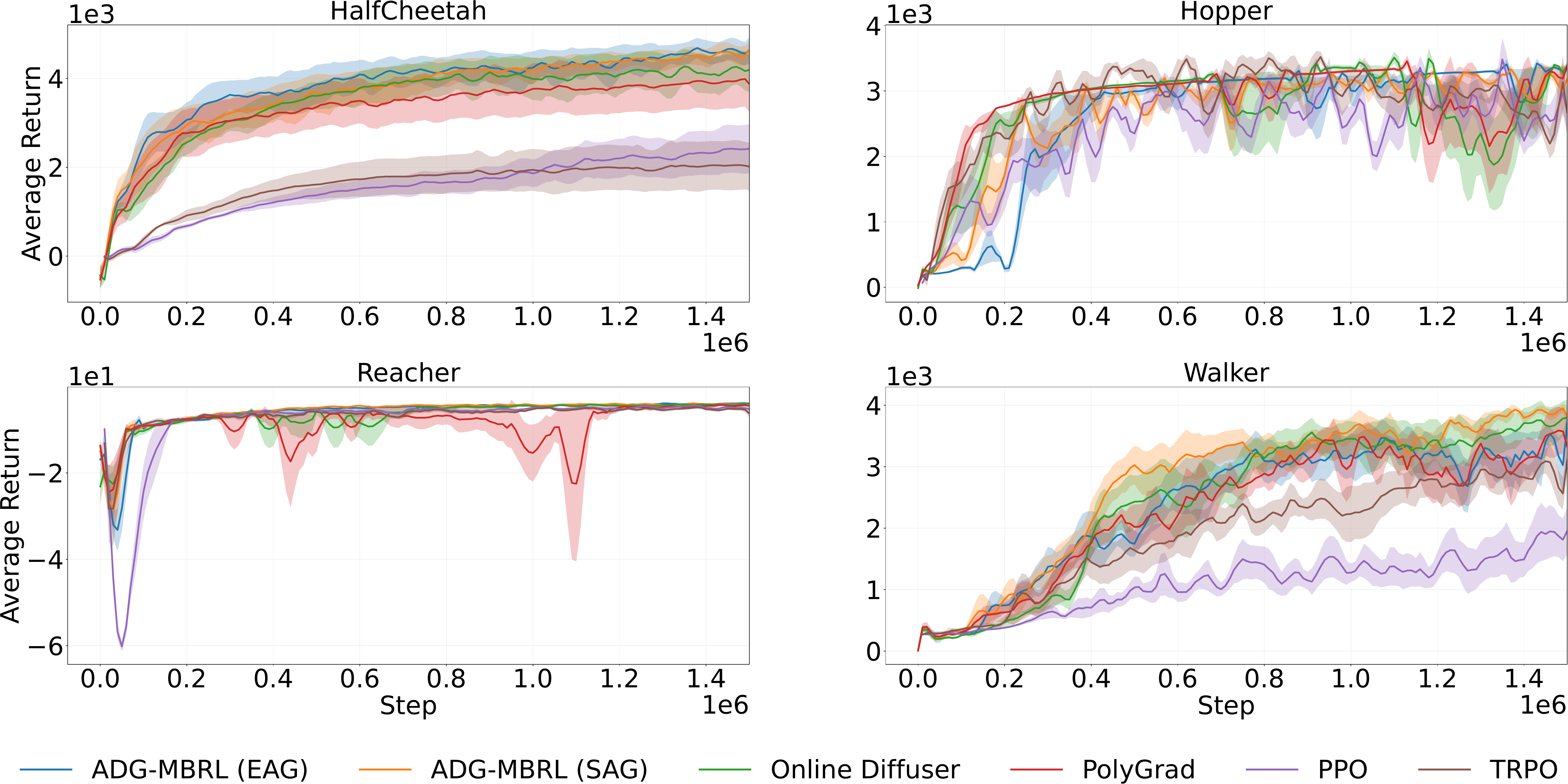}
    \caption{Training curves for the Mujoco environments: HalfCheetah, Hopper, Walker and Reacher. Shaded areas indicate the standard error. The curves have been smoothed by a running average for better readability.}
    \label{fig:results}
\end{figure*}

Our implementation builds directly on top of PolyGRAD, adding the relevant guides by adding only a few lines of code. We follow the same training procedure and use the same MLP-based diffusion model, trained by minimizing the $L_2$ loss from (\ref{eq:diffusion_loss}) with ADAM as optimizer. The MDP uses a learnable embedding of the diffusion step $i$, which is common for diffusion methods \cite{janner2022planning}. We implemented our version of Online Diffuser in the same way, changing the diffusion guide accordingly.  For TRPO and PPO we used the implementation from Stable-Baselines3 \cite{stable-baselines3}, with their fine-tuned hyperparameters.

\textbf{Results and analysis.} Table \ref{tab:mujoco_results_se} presents the final performance of the guides and the comparison methods. For HalfCheetah, Walker, and Reacher, we observe a clear pattern: AGD-MBRL constantly outperforms all the other baselines, sometimes by a large margin (as in HalfCheetah). PolyGRAD serves as a solid baseline and outperforms model-free methods. This is expected since model-based methods are more sample efficient than their model-free alternatives, while requiring increased computational resources. The Diffuser guide helps improving the PolyGRAD performance, indicating that generating trajectories with a higher cumulative reward helps the agent to converge to a more optimal solution. However, AGD-MBRL is able to learn policies that lead to higher cumulative reward, showing that leveraging the learned value function leads to better planning for the RL agent. In the Hopper environment we see that the advantage-based guides perform slightly worse than PolyGRAD. We believe this is due to stochasticity, as Hopper is the easiest environment among the ones we chose, where all algorithms, even the model-free ones, seem to converge to similarly performing policies. When closely observing the results of the two EAG and SAG guides, we can see that in a task like HalfCheetah EAG is able to significantly outperform SAG. We believe this is a confirmation of the analysis provided at the end of section \ref{sec:Advantage_guided_diff}: EAG greatly increases the sampling frequency of trajectories with a high advantage, while SAG is more conservative, due to the bounded sigmoid function. In an environment like HalfCheetah, where the value function is realatively easy to estimate and exploit, EAG is expected to quickly reach a better performance than SAG. However, in environments where the optimal value function is harder to approximate, like Walker, a more conservative sampling leads SAG to outperform EAG already in the early training stages. This can be clearly observed in Figure \ref{fig:results}, where we provide the training curves for all methods in the considered environments. In Figure~\ref{fig:results}, AGD-MBRL shows more stable learning than the other diffusion-based baselines. On Hopper and Reacher, both EAG and SAG reduce the magnitude and frequency of performance regressions. A plausible explanation is that advantage-guided sampling induces optimistic yet directed exploration: it biases the generating process toward states and actions with high estimated value while updating the advantage estimate by collecting new data, keeping synthetic rollouts aligned with the on-policy improvement direction. By contrast, PolyGRAD generates trajectories aligned with the current policy but ignores value information, so it will sample low-advantage actions more often and suffer more instability. Diffuser-style reward guidance helps, but is affected by its short-horizon: without accounting for the policy value it increases the risk of suboptimal exploration.

\begin{table}[h!]
\centering
\caption{Final training return for MuJoCo continuous-control tasks.}
\resizebox{\linewidth}{!}{
\begin{tabular}{@{}lrrrr@{}}
\toprule
 & \textbf{Hopper} & \textbf{HalfCheetah} & \textbf{Walker} &
 \textbf{Reacher}\\ \midrule
AGD-MBRL (EAG) &
3328.61 $\pm$ 39.45 &
\textbf{4864.80 $\boldsymbol\pm$ 156.83} &
\textbf{3794.24 $\boldsymbol\pm$ 165.03}&
\textbf{ -3.90 $\boldsymbol\pm$ 0.13}\\
AGD-MBRL (SAG) &
3268.42 $\pm$ 24.68&
\textbf{4642.25 $\boldsymbol\pm$ 162.52} &
\textbf{3844.67 $\boldsymbol\pm$ 361.58}&
\textbf{-3.87 $\boldsymbol\pm$ 0.37}\\
PolyGRAD + Diffuser Guide &
\textbf{3353.97 $\boldsymbol\pm$ 25.30} &
4188.43 $\pm$ 489.66 &
3777.73 $\pm$ 294.98 &
-3.96 $\pm$ 0.40\\
PolyGRAD &
3346.99 $\pm$ 52.39 &
$3879.16 \pm 626.40$ &
$3489.48 \pm 456.70$ &
$-4.48 \pm 0.13$\\
PPO &
$2998.90 \pm 432.28$ &
$2408.20 \pm 546.33$ &
$1894.03 \pm 349.06$ &
$-5.17 \pm 0.57$\\
TRPO &
$3270.27 \pm 273.04$ &
$2014.91 \pm 539.64$ &
$3090.80 \pm 267.79$ &
$-6.22 \pm 0.85$\\
\bottomrule
\end{tabular}}
\vspace{0.1cm}
\label{tab:mujoco_results_se}
\vspace{-0.5cm}
\end{table}

\section{Conclusions and future work}

We introduced Advantage-Guided Diffusion as a way to enhance trajectory generation in MBRL by guiding the diffusion generative process toward regions of high long-term reward. We showed that the proposed guides, SAG and EAG, can significantly improve the performance of existing frameworks on complex optimal control tasks. As we discussed in Section \ref{sec:Advantage_guided_diff}, other functions could be used to steer the generative diffusion process.  We believe that an exciting future research avenue lies in exploring different guiding functions that might lead to additional benefits for the agent learning process.

The key limitation with using diffusion models for generating trajectories is the generation time. Diffusion models are computationally expensive, as the iterative forward process for generating trajectories requires a significant number of generative steps. Interesting future research directions include how to make generation faster, for example by generating samples in a latent space \cite{rombach2021high} or using flow matching \cite{lipman2022flow}.






\section*{Acknowledgements}
The computations were enabled by resources provided by the National Academic Infrastructure for Supercomputing in Sweden (NAISS), partially funded by the Swedish Research Council through grant agreement no. 2022-06725. This work was supported in part by Swedish Research Council Distinguished
Professor Grant 2017-01078, Project Grant 2024-05043, Knut and Alice Wallenberg Foundation Wallenberg Scholar Grant, and Swedish Strategic Research Foundation FUSS SUCCESS Grant. 


\bibliographystyle{ieeetr}
\bibliography{ref}
\end{document}